\pdfoutput=1
\documentclass[conference]{IEEEtran}
\IEEEoverridecommandlockouts
\usepackage{amsmath,amssymb,amsfonts,amsthm}
\usepackage{graphicx}
\usepackage{booktabs}
\usepackage{stfloats}   % allows [b] placement of double-column floats (IEEEtran HOWTO)
\usepackage{multirow}
\usepackage{xcolor}
\usepackage{algorithm}
\usepackage{algpseudocode}
\graphicspath{{figs/}}
\providecommand{\newblock}{\hskip .11em plus .33em minus .07em}

\newtheorem{lemma}{Lemma}
\newtheorem{proposition}{Proposition}
\newtheorem{remark}{Remark}

\newcommand{\R}{\mathbb{R}}
\newcommand{\free}{\mathcal{F}}
\newcommand{\lib}{\mathcal{U}}
\newcommand{\gcsstar}{GCS$^*$}
\newcommand{\ixg}{IxG$^*$}

\title{\bf Sampling-based Certified Planning with\\
Graphs of Convex Sets}

\newif\ifanon
\anonfalse

\ifanon\else\usepackage[hidelinks]{hyperref}\fi

\author{Peng~Xie and Amr~Alanwar%
\thanks{Peng Xie and Amr Alanwar are with the TUM School of Computation, Information and Technology, Department of Computer Engineering, Technical University of Munich, 74076 Heilbronn, Germany (e-mail: p.xie@tum.de; alanwar@tum.de).}}

\begin{document}
\maketitle

\begin{abstract}
Planners on graphs of convex sets return trajectories that are
collision-free by construction, provided the convex regions are
collision-free.  The region generator only promises that property
probabilistically, and no planner in the family verifies it.  We report
the first measurement of what the gap costs.  On a scaled 14-DOF bimanual
library, $3.2\%$ of interface samples are in collision, and a search-based
GCS planner (\gcsstar) turns that volume error into a $62\%$ answer error:
$18$ of $29$ pick-and-place queries return trajectories that drive the
arms through the shelves, up to $91$\,mm deep, reported as successes.
Repairing the library does not work; a ten times stricter acceptance
contract, sums-of-squares certified regions, and uniform margins each
destroy the connectivity planning needs before they deliver soundness.
We instead build a planner that certifies its answers.  It samples the
overlaps and shared faces of the decomposition, prunes with an admissible
informed bound, and verifies the one candidate each search round proposes,
continuously, by a chain of clearance certificate balls with no resolution
parameter; failures are repaired with local in-region detours, and the
convex polish is re-verified.  Head-to-head on all $29$ task queries it
delivers zero invalid answers against $21$ for the reference, reaches its first certified answer in $0.11$\,s against $1.59$\,s for
the reference's unverified one, and reproduces the reference optimum
exactly on every query whose reference answer is physically valid.
\end{abstract}

% Introduction -- v2, 2026-08-07.  Plain language, one page.  Formal
% definitions live in Sec. Problem Statement; the literature review lives in
% Related Work.  Numbers carry "% src:" comments naming their artifact.
\section{Introduction}

Planners built on graphs of convex sets (GCS) return trajectories that are
collision-free by construction: the trajectory stays inside a library of
convex regions, and the regions are collision-free.  The second half of that sentence is an assumption inherited from the
region generator.  The original formulation~\cite{marcucci2023sciencerobotics} and its
search-based descendants~\cite{chia2024gcsstar} take it as given, and
neither re-examines the trajectory returned.

The generators that produce these libraries at scale make a probabilistic
promise: after a batch of random samples inside a candidate region comes back
clean, the region is accepted, and its colliding volume is small with high
confidence~\cite{werner2024iriszo}.  Small is not zero.  On a 14-DOF bimanual library built at default settings, $3.2\%$ of
the samples we draw on region overlaps are in collision.
% src: e7_biman219.json (64/2000); e7_panda55.json gives 0.3% on a curated 7-DOF library
Feeding that library to \gcsstar~\cite{chia2024gcsstar} turns a
$3.2\%$ model error into a $62\%$ answer error: $18$ of $29$
pick-and-place queries return trajectories that drive the arms through
the shelves, by $39\,$mm at the median and $91\,$mm at the worst, for
up to half the length of the path.
% src: e9_task_sweep_verdicts.json (18/29); e9b_depth_profiles.json (depths)
It reports success on every one.

\begin{figure}[t]
\centering
\includegraphics[width=0.85\linewidth]{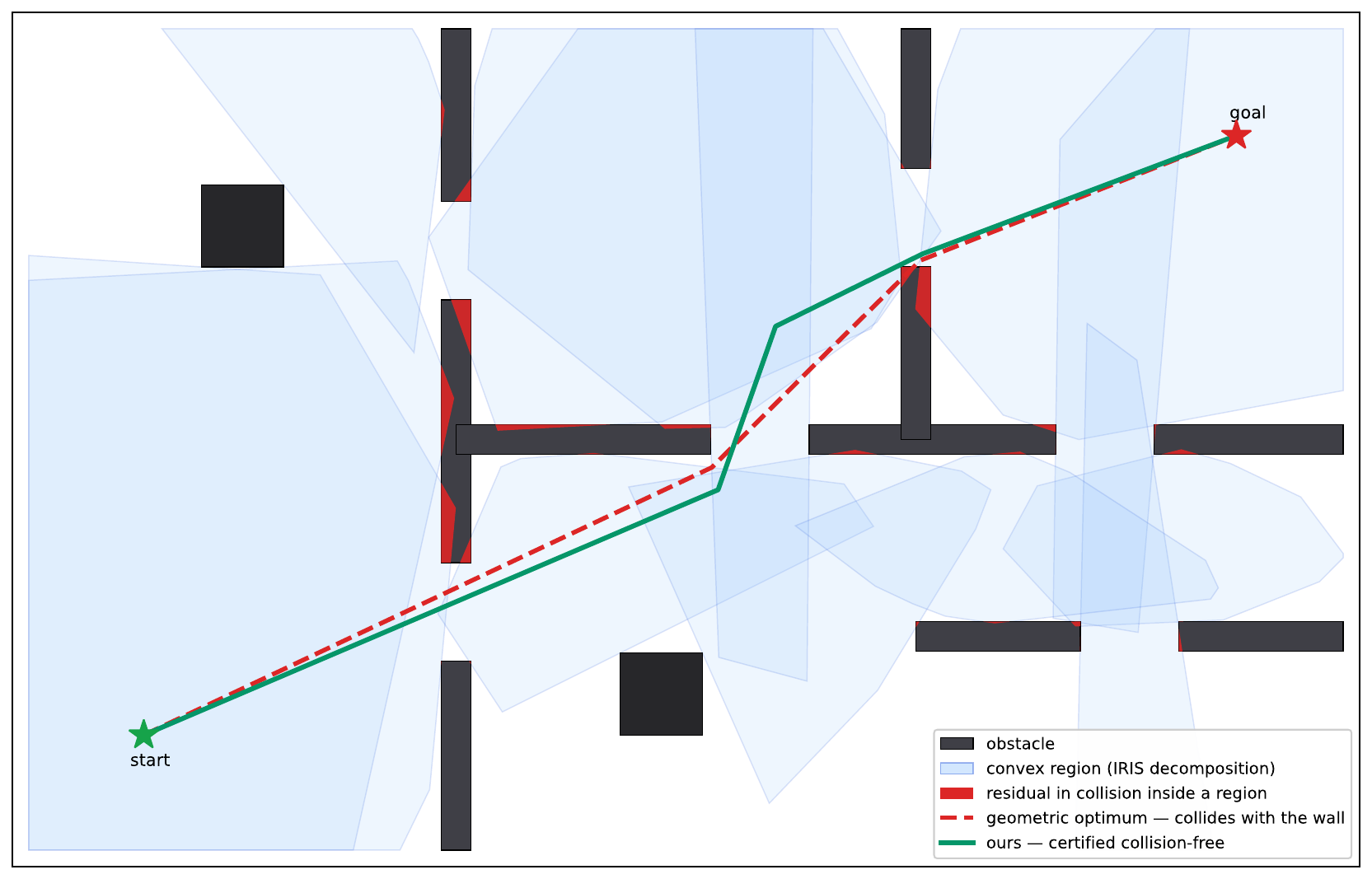}
\caption{Why a task fails.  The decomposition tolerates thin residual wedges
(red), and the reference optimum (dashed red) presses into one; our
planner returns the certified path (green).
Sec.~\ref{sec:problem} derives the mechanism.
% src: e13c_fig2d_floorplan.py
}
\label{fig:floorplan}
\end{figure}

A planner is an optimizer, and that is why the error grows on the way
through.  It searches for the cheapest trajectory the model allows: wherever
optimism shortens the path, the optimum moves there.  Shortest paths also hug obstacles, exactly where a
sampling-based generator's residual sits (Fig.~\ref{fig:floorplan}), and
grasp configurations put the endpoints in the same cramped geometry.  

Tightening the library does not repair it: a ten times stricter
acceptance contract, sums-of-squares certified
regions~\cite{dai2024ciris}, and uniform margins each dissolve the
region graph before they deliver soundness (Sec.~\ref{sec:ciris}).
% src: e10_regen219_report.json; e15_ciris.json; e8_erosion_biman219.json

Our planner buys certainty along the trajectory instead, and it buys it by
sampling.  We sample the overlaps and shared faces of the library, the places a
trajectory must cross to move between regions, and check every sample against the true collision geometry before it joins the roadmap.
Sampling puts the candidate set in our hands: a waypoint is a
configuration that can be tested, rejected, or moved; an optimum of a
convex program offers no such handle.  Sampling also makes the planner
fast: an informed filter discards interfaces that cannot improve the
incumbent, and the convex program runs once on a short corridor.
A candidate is then verified continuously, one clearance query
certifying a whole ball of configurations and overlapping balls
certifying the segment with no discretization step to fall through;
failed segments are repaired inside the region that produced them,
where convexity keeps the detour legal, and the cone program that
polishes the survivor goes back through the same verification.

The result is a planner that hands back only verified answers.  Across
all $29$ task queries the reference planner delivers $21$ trajectories that
fail certification and ours delivers zero, our first certified answer arrives in a fourteenth of the time the
reference needs for an unverified one, and on every query whose
reference answer is physically valid we return that optimum exactly.
% src: e16_headtohead.json

\paragraph{Contributions}
This paper reports the first measurement of region soundness in GCS
planning and of its effect on planner output, through a benchmark protocol built for bit-exact, twice-validated,
cross-machine reproduction.  It then
gives evidence that the defect cannot be fixed in the library, since a
stricter acceptance contract, sums-of-squares certificates, and uniform
margins each dissolve the region graph before they make it sound.  Its
constructive part is a sampling-based planning layer, combining
interface sampling, informed pruning, continuous clearance certificates,
in-region repair, and a re-verified convex polish, that keeps the graph
intact, returns certified paths or explicit failures, and matches the
reference optimum wherever the reference is right.  The code, both
libraries, the benchmark data, and the result videos are released.

% Related Work -- v1, 2026-08-07.  Positive-first: each work is introduced
% by what it contributes, then by how the region assumption enters it.
% Citations are attached individually to the work they name.
\section{Related Work}

This section places the paper in three bodies of work: the GCS planning
family whose shared assumption we measure, the generators that produce
its regions, and the checking and certificate machinery our layer
builds on.

\subsection{Planning on convex-region libraries}

Marcucci et al.~\cite{marcucci2024sppgcs} formulate shortest paths in
graphs of convex sets with a relaxation tight enough that rounded
solutions are near-optimal, turned into a motion planner
in~\cite{marcucci2023sciencerobotics}.  A family followed: implicit
search (\gcsstar~\cite{chia2024gcsstar};
\ixg~\cite{natarajan2024ixg}), multi-query bounds and
walks~\cite{morozov2024multiquery,morozov2025swp},
travelling-salesman variants (GHOST~\cite{tang2026ghost}),
tensor-train compression~\cite{tango2026}, and geodesic
generalization~\cite{cohn2025ggcs}; planners over safe
boxes~\cite{marcucci2024fpp} and optimized
covers~\cite{wu2025cover} rest on the same input.

These planners share one assumption: the regions they consume are
collision-free.  Table~\ref{tab:trail} records how five of them state
it; the rest consume the decomposition as given, and none checks the
returned trajectory against the collision geometry.

\begin{table}[t]
\centering
\scriptsize
\setlength{\tabcolsep}{3pt}
\caption{How the region assumption enters five representative works.
Quotations are verbatim.}
\label{tab:trail}
\begin{tabular}{p{1.6cm}p{5.9cm}}
\toprule
work & treatment of ``regions are collision-free'' \\
\midrule
GCS motion planning~\cite{marcucci2023sciencerobotics} &
one sentence: the fast generator ``does not provide a rigorous
certification, but \ldots{} appears to be very reliable in practice'' \\
GCS$^*$~\cite{chia2024gcsstar} &
guarantees stated relative to the mathematical program; regions are input \\
IxG$^*$~\cite{natarajan2024ixg} &
sets ``capture the free and safe planning space'' \\

Shortest walks~\cite{morozov2025swp} &
``we produce a convex decomposition of the collision-free configuration
space'' \\

GGCS~\cite{cohn2025ggcs} &
plans confined to ``certified'' C-Free (quotes in the original), grown by
the uncertified generator \\
\bottomrule
\end{tabular}
\end{table}

\subsection{Convex region generation}

IRIS~\cite{deits2014iris} inflates a convex region around a seed by
separating it from obstacles with hyperplanes: exact against convex workspace
obstacles but not against their non-convex C-space preimages
(Sec.~\ref{sec:problem}).  Two branches answer this.  In the
statistical branch, IRIS-NP~\cite{petersen2023irisnp} accepts a region
after a run of failed counterexample searches, a probabilistic
certificate with no explicit bound, and
IRIS-ZO~\cite{werner2024iriszo} formalizes the test, bounding the
colliding volume fraction by $\varepsilon$ at confidence $1-\delta$;
C-IRIS~\cite{amice2022ciris,dai2024ciris} instead proves separation
with sums-of-squares programs.  The planning literature runs on the
statistical branch, which builds 7- and 14-DOF libraries in minutes;
the proofs cost minutes to hours per region at 7--12
DOF~\cite{dai2024ciris}, and Sec.~\ref{sec:ciris} measures what
happens when they feed a planner.

\subsection{Collision checking and feasibility certificates}
\label{sec:related-cert}

A distance query returns more than a yes/no answer: the clearance at a
configuration certifies a ball of collision-free configurations around
it~\cite{bialkowski2016certificates}, and roadmaps can be built from such
balls~\cite{lee2024bubble}, also with learned configuration-distance
models~\cite{wullt2026pbrm}.  Lazy planners search first and validate only
the path they intend to return~\cite{bohlin2000lazyprm}, a scheme with
formal edge-selection guarantees~\cite{dellin2016lazysp}.  Our planner puts both to work on a convex decomposition, whose overlaps
tell us exactly where a trajectory must be checked and whose convexity
turns a failed segment into a repairable one.  The informed anytime
planners~\cite{gammell2015bit,strub2022aitstar} sample free space
itself, build a fresh graph per query, and approach the optimum only
asymptotically, their validity resting on resolution edge checks; we sample the decision set of a precomputed decomposition, reach the
corridor optimum with one convex solve carrying a per-answer gap, and
certify the continuum.  In 2- and 3-D workspaces, where cells can be
carved exactly, decomposition planners already win this comparison
outright: neural corridor selection reports $99$ to $100\%$ success at
$2$ to $280\times$ the speed of sampling baselines~\cite{xie2026gnndip,xie2025hzmp}.  Configuration space is where the paradigm's
premise breaks, the cells turning probabilistic
(Sec.~\ref{sec:problem}); this paper carries the paradigm there with
its guarantees intact.

Existing certificates describe the program (relaxation gaps, a
bounded-suboptimality factor in GHOST~\cite{tang2026ghost}) or the
regions
(C-IRIS); ours describes the answer: an $\varepsilon$-gap on cost and a
continuous clearance certificate on the path (Sec.~\ref{sec:layer}).

% Problem Statement -- v2, 2026-08-07.  IRIS decomposition first, then the
% "why the residual is unavoidable" derivation with fig_residual, then the
% consequence and the requirement on an answer.
\section{Problem Statement}
\label{sec:problem}

This section formalizes the gap between the modeled and the physical
problem, derives why the residual must exist, and compresses its
behavior into three testable predictions.

% fig:residual is encountered here (page 2) so that it lands at the TOP of
% page 3 (double-column [t] floats go to the top of the following page).
\begin{figure*}[!t]
\centering
\includegraphics[width=0.82\textwidth]{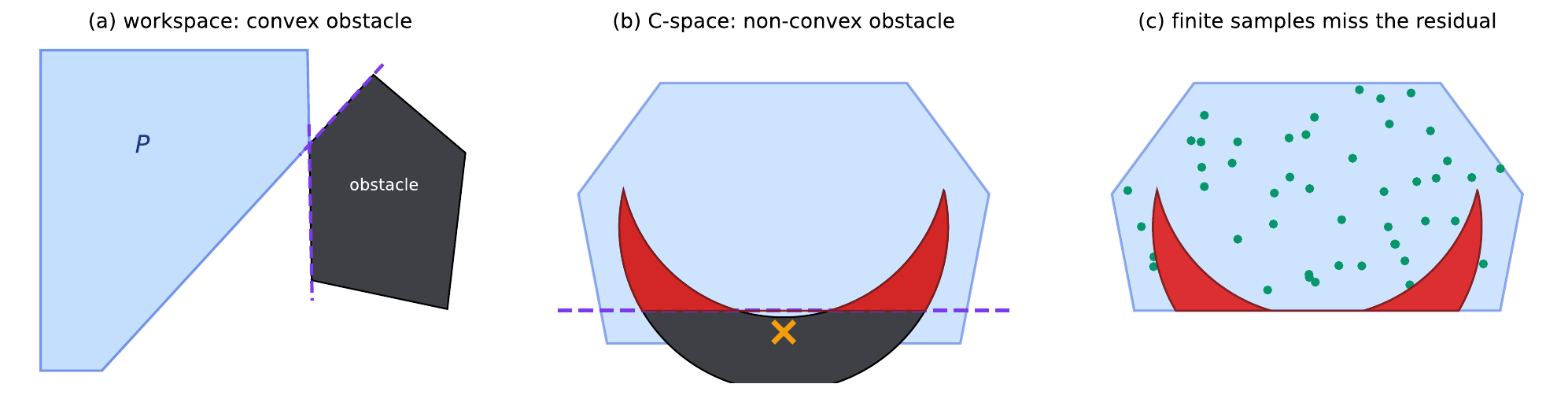}
\caption{Why the residual is unavoidable.  (a)~Against a convex obstacle,
cuts placed on its own faces are a complete and exact exclusion proof: the
free region hugs the obstacle.  (b)~The C-space obstacle is non-convex: a cut at the discovered
collision ($\times$) removes a half-space, and the parts that curve
back around it (red) remain inside the region.  (c)~Uniform samples
(green) miss the thin residual, so the statistical test accepts.
% src: e14_fig_why_pockets.py
}
\label{fig:residual}
\end{figure*}

\paragraph{Setting and decomposition}
Let $\mathcal{Q}\subset\R^{d}$ be the configuration space of a robot and
$\free\subset\mathcal{Q}$ its collision-free subset, decided by a collision
checker that we treat as ground truth.  A query is a pair $(q_s,q_g)$ of
collision-free configurations; a path is a continuous curve
$q:[0,1]\to\mathcal{Q}$ between them, with arc length as cost.

An IRIS-style generator~\cite{deits2014iris} grows a convex polytope around a
seed configuration by alternating two steps: inflate an ellipsoid inside the
current polytope, and add a separating hyperplane, a \emph{cut}, that pushes the
polytope away from an obstacle.  Grown from many seeds, the
regions form a library $\lib=\bigcup_{i=1}^{n}X_i$.  Two regions are
adjacent when their \emph{interface} $X_u\cap X_v$, itself a polytope, is
nonempty; the regions and interfaces form the graph the planner searches.
The planner solves the \emph{modeled} problem
\begin{equation}
\label{eq:modeled}
m(\lib)\;=\;\min\{\,\mathrm{len}(q)\ :\ q(t)\in\lib\ \ \forall t\,\},
\end{equation}
and the user needs the \emph{physical} one, with $\lib$ replaced by
$\lib\cap\free$.  The two coincide exactly when $\lib\subseteq\free$.

\paragraph{Non-convexity and the acceptance residual}
In the workspace with convex obstacles, cuts finish the proof: a convex
region disjoint from a convex obstacle admits a strictly separating
hyperplane, and placing the cuts on the obstacle's own faces excludes it exactly; the free
region hugs the obstacle with zero slack
(Fig.~\ref{fig:residual}a).  In configuration space the hypothesis fails.  For a convex
workspace body $W$ and robot geometry $B(q)$ posed by forward kinematics,
the C-obstacle is the preimage
\begin{equation}
O=\{\,q:\ B(q)\cap W\neq\emptyset\,\},
\end{equation}
and forward kinematics makes $O$ in general non-convex.  No hyperplane
separates a convex region from it; a cut placed at a discovered colliding
configuration removes only a supporting half-space, and the parts
of $O$ that curve back around the cut remain inside the region
(Fig.~\ref{fig:residual}b).  Finitely many cuts cannot exclude a non-convex set in general; a
universal proof needs the algebraic certificates of
C-IRIS~\cite{dai2024ciris}, whose cost is why scaled pipelines do not
run it.

Unable to prove, the generator tests: accept the region after $M$ clean
uniform samples~\cite{werner2024iriszo}.  The mathematical content is the
Bernoulli bound
\begin{equation}
\label{eq:contract}
\begin{split}
&\Pr[\text{accept}\mid\text{colliding fraction}>\varepsilon]\\
&\quad\le\;(1-\varepsilon)^{M}\;\le\;e^{-\varepsilon M}\;\le\;\delta
\quad\text{for } M\ge\tfrac{\ln(1/\delta)}{\varepsilon},
\end{split}
\end{equation}
so acceptance certifies a colliding fraction at most $\varepsilon$ at
confidence $1-\delta$; a zero-$\varepsilon$ contract needs
$M=\infty$.  Shipped defaults are $\varepsilon=1\%$,
$\delta=5\%$, and an iteration cap after which regions are admitted with
the test unfinished.  We call $\lib\setminus\free$ the \emph{residual}; it is what the test
tolerates: thin in volume, so the samples miss it
(Fig.~\ref{fig:residual}c).  The residual is not a breach of the contract;
it is the contract.

% fig:dimlaw (cited in Remark 1 below) is encountered here, early in the second
% column of page 3, so that it is placed at the top of that column under Fig. 2.
\begin{figure}[!t]
\centering
\includegraphics[width=0.85\linewidth]{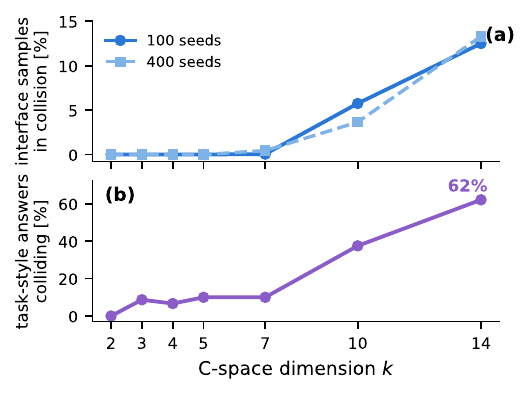}
\caption{Residual against C-space dimension.  (a)~Interface contamination on the planar
$k$-link family (fixed workspace, obstacles, and generator; only $k$
varies) at two seed budgets.  (b)~Task-style answers colliding:
stress-endpoint queries on the same family through $k{=}10$; the 14-DOF
point is the bimanual task benchmark ($18/29$).  Through $k{=}5$ the volume statistic reads exactly zero while
delivered answers already collide.
% src: e32_fig_dimlaw.py; out/e31/b100/*.json; out/e31/ans_k*.json
}
\label{fig:dimlaw}
\end{figure}

\paragraph{Boundary concentration and optimizer exploitation}
Three effects put the tolerated volume in the planner's way.
(i)~Cuts are tangent at discovered collisions, so the undiscovered slivers
hug the polytope faces: the residual lives in the boundary shell.
(ii)~High-dimensional volume concentrates there: shrinking a convex body
in $\R^{d}$ by $5\%$ leaves $0.95^{d}$ of its volume, so at $d=14$ the
shell holds $1-0.95^{14}\approx 51\%$ of the body.
(iii)~Interfaces $X_u\cap X_v$ are double boundary zones, the
meeting of two obstacle-hugging shells, and they are exactly where paths
cross.  This is why a measurement on interfaces can read
$3.2\%$ against a $1\%$ volume contract without contradiction: it is the
operationally relevant measure of the same residual.

\begin{proposition}
\label{prop:exploit}
If $m(\lib)<m(\lib\cap\free)$, then every optimizer
of~\eqref{eq:modeled} is physically invalid.
\end{proposition}
\begin{proof}
A path attaining $m(\lib)$ costs less than $m(\lib\cap\free)$, and every
path inside $\lib\cap\free$ costs at least $m(\lib\cap\free)$.
\end{proof}

\noindent
The hypothesis holds whenever the residual offers a shortcut, and by the
concentration above it usually does: minimum-length paths press against
the same shell, and grasp endpoints sit where the generator's budget
runs out first.

\begin{remark}[Dimension dependence of the residual]
\label{rem:dimlaw}
Fix the acceptance contract $(\varepsilon,\delta)$ and the seed
budget.  The tolerated residual keeps a volume share of at most
$\varepsilon$ in every dimension $d$, but effects (i)--(iii) confine
it to the boundary shell of each region, and the share of a convex
body within a relative distance $\rho$ of its boundary grows as
$1-(1-\rho)^{d}$.  The residual's share of the interface measure, the
measure a path has to cross, therefore grows with $d$ while the
reported volume statistic does not, and by
Proposition~\ref{prop:exploit} an optimizer selects that residual
whenever it shortens the path: the probability that a delivered answer
collides rises with $d$ even where uniform sampling reports no
contamination.  Fig.~\ref{fig:dimlaw} isolates the effect on a planar
$k$-link arm with everything but $k$ fixed: interface contamination is
zero through $k{=}5$ and $13\%$ at $k{=}14$, while task-style answers
already collide from $k{=}3$ on.
\end{remark}

\paragraph{Predictions}
Three consequences follow, and Section~\ref{sec:measure} tests each:
the dimension dependence of Remark~\ref{rem:dimlaw}; the amplification from volume to answers,
whose error exceeds the residual's volume share by an order of
magnitude and grows as query endpoints approach obstacles, with
Proposition~\ref{prop:exploit} as its kernel; and the failure of
library-level tightening, which removes interfaces before it removes
the residual, because certainty about a $d$-dimensional volume is paid
wherever the volume meets obstacles, where regions must also meet each
other.

% fig:interface (cited in Sec. IV-A) is encountered here, in the second column
% of page 3, so that it is deferred to the top of page 4 next to its citation.
\begin{figure*}[!t]
\centering
\includegraphics[width=0.9\textwidth]{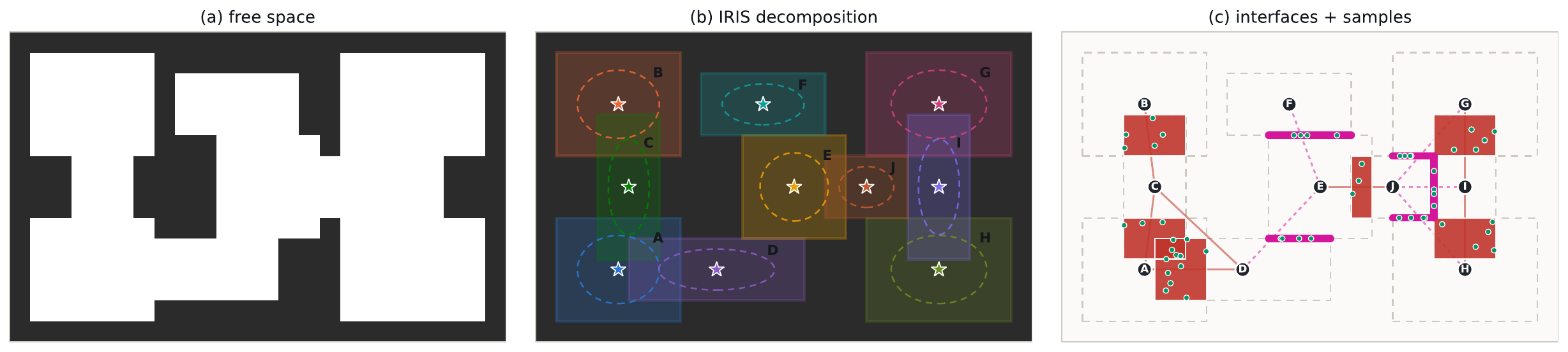}
\caption{From free space to a sampled decision set.  (a)~The free space of
a floor plan.  (b)~Its IRIS decomposition: seeds (stars) inflate into
overlapping convex regions.  (c)~The interfaces the decomposition
induces: volume overlaps (red) and shared faces (magenta), with uniform
samples (green) drawn on every interface; nodes and thin lines sketch the
region graph they populate.
% src: e22_fig_decomp_paper.py
}
\label{fig:interface}
\end{figure*}

\paragraph{Certified answers}
A planner consuming a library that satisfies only~\eqref{eq:contract} should
return, for each query, either a path $q$ with a proof that
$q(t)\in\free$ holds for every $t\in[0,1]$, or an explicit report that no such path was found.  The
proof must cover the continuum, because the residual is thin exactly
where paths run.  The next section builds a
planner with this property.

% Methodology, part 1 -- sampling on overlaps and shared faces.
% Figure: fig_decomp.pdf (generated by stageA/e22_fig_decomp_paper.py).
\section{Method: Sampling and Certification}
\label{sec:sampling}

This section builds the planner: sampling on the interfaces, informed
iteration toward the optimum, and the certification layer that proves
every delivered answer.

\subsection{Interface sampling and roadmap}

Inside one convex region the cheapest way between two configurations is the
straight segment between them, and that segment stays inside the region.  A
path that visits a sequence of regions $X_{v_1},\dots,X_{v_k}$ at minimum
length is therefore piecewise linear, and its vertices lie on the interfaces
$X_{v_i}\cap X_{v_{i+1}}$.  The interfaces hold every free variable of the modeled
problem~\eqref{eq:modeled}, and sampling them discretizes the decision
itself rather than the space around it.

Interfaces come in two shapes, both polytopes
(Fig.~\ref{fig:interface}).  A \emph{volume overlap} is $d$-dimensional: the
two regions share a body of configurations, and a crossing may sit anywhere
inside it.  A \emph{shared face} is $(d\!-\!1)$-dimensional: the regions meet
on a facet, and every crossing lies on that facet.  Contacts of lower dimension (an edge, a corner, a pinch between two
obstacles) also satisfy $X_u\cap X_v\neq\emptyset$, and admitting them as
transitions makes the roadmap claim passages that no robot can execute:
two regions can meet at a single point wedged between obstacles, and a
segment through that point is legal in the model with zero clearance in
the world.
We therefore require
\begin{equation}
\label{eq:traversable}
\dim\!\left(X_u\cap X_v\right)\;\ge\;d-1
\end{equation}
before an interface enters the roadmap.
% src: e6 step-1 verification, membership/traversability study (12.73 vs 13.77 without the rule)

Each interface is available as $\{x: A_ux\le b_u,\ A_vx\le b_v\}$.  We detect
the rows that hold with equality throughout, which gives the affine hull of
the interface and its intrinsic dimension, and we sample in hull
coordinates.  Four samplers cover the shapes a crossing can take.

\emph{Center} takes the Chebyshev center of the interface, the crossing
that stays furthest from its rim.  \emph{Area} draws uniform samples in the
interface, spreading the crossings over the full set of options.
\emph{Contour} samples the relative boundary, where the shortest crossing
presses when a corridor turns.  \emph{Contour-in} pulls the contour inward
by $\lambda$, keeping samples on the facet proper; a raw contour point
sits on a lower-dimensional face, exactly the geometry
that~\eqref{eq:traversable} rules out.  All four deliver the same
certified coverage on the closing benchmark, and center crossings,
maximal-clearance points by construction, certify fastest; they anchor
the head-to-head of Sec.~\ref{sec:h2h}.
% src: out/e24 ablations (15/29 for all modes; center t_first 0.11s)
Every sample is a configuration that can be tested: the collision
checker keeps the clean ones, and the rest never enter the roadmap.

\paragraph{Roadmap construction}

Nodes are the accepted samples together with $q_s$ and $q_g$.  Each node
records the regions that produced it (an interface sample records both
parents), and two nodes are joined when they share a region, with
Euclidean distance as the edge cost.  Membership is carried, not re-derived:
a node inherits the regions that produced it, and duplicate nodes merge by
taking the union of their region lists.
% src: e6 declared-membership study (dedup must union parents; geometric membership lets paths cross thin walls)

\subsection{Informed iteration}
\label{sec:informed}

The planner runs in rounds and improves monotonically
(Fig.~\ref{fig:method}, Alg.~\ref{alg:planner}).  Round one samples every
interface with a small budget, searches the roadmap, and produces the first
incumbent path.  The incumbent's length $c$ bounds what any better path may do.  A path through interface $F$ is at least as
long as
\begin{equation}
\label{eq:ell}
\ell(F)\;=\;\min_{x\in F}\;\lVert x-q_s\rVert+\lVert x-q_g\rVert ,
\end{equation}
so every interface with $\ell(F)>c$, outside the informed ellipse of
Gammell et al.~\cite{gammell2014informed} with foci $q_s,q_g$ and major
axis $c$, is off every improving path, and discarding it is lossless;
batch and adaptive descendants~\cite{gammell2015bit,strub2022aitstar}
refine the same device for sampling planners.  The next round doubles the sample budget and spends it only
on the survivors: the search volume shrinks as
the incumbent tightens.  Disabling the filter changes no answer
and multiplies total planning time by $2.2$; the pruning is pure speed.
% src: out/e24 inf0 runs (26s vs 12s median total)  Each round re-runs the search on
the denser roadmap, and the incumbent only ever improves.  In
Fig.~\ref{fig:method} the first round costs $6.61$; its ellipse retires the detour over the
top of the map, and the resampled round returns $5.28$.

\begin{figure*}[t]
\centering
\includegraphics[width=0.9\textwidth]{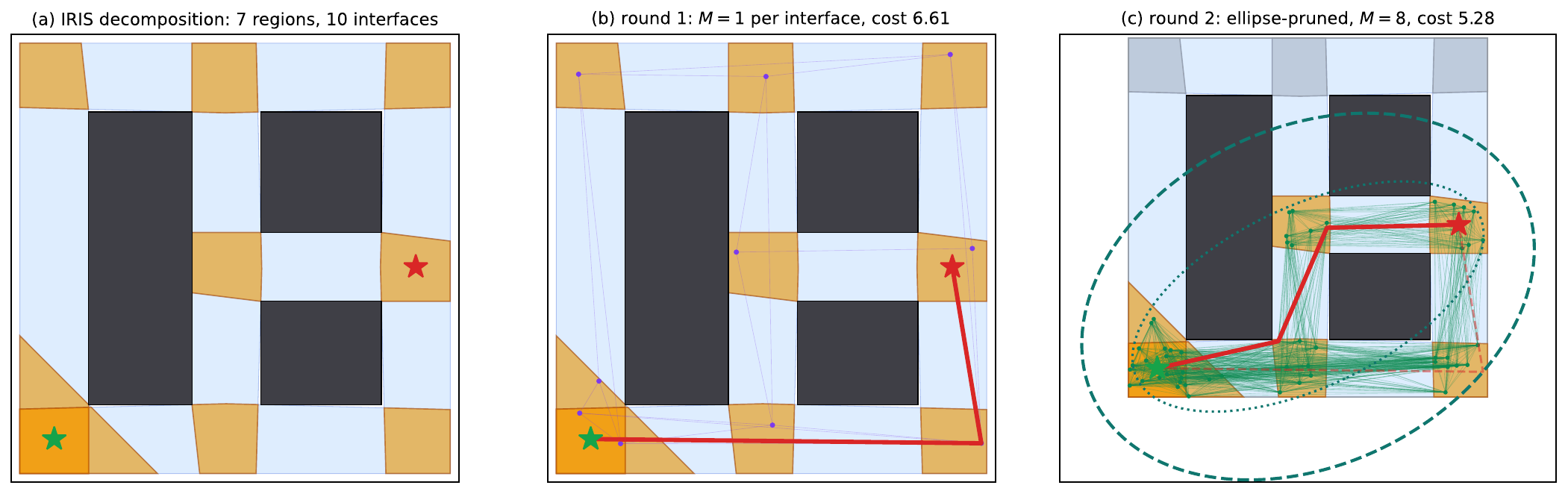}
\caption{The planner on a 2-D library (all elements computed).  (a)~IRIS
decomposition; orange sets are the interfaces.  (b)~Round one: one sample
per interface, Dijkstra on the roadmap, first incumbent.  (c)~The incumbent defines the informed ellipse (dashed); interfaces
outside it (gray) leave the search, the survivors are resampled more
densely, and the path improves; the tighter incumbent yields the
smaller dotted ellipse for the next round.
% src: e21_fig_method.py
}
\label{fig:method}
\end{figure*}

\begin{algorithm}[t]
\caption{Interface-sampling planner (one query)}
\label{alg:planner}
\begin{algorithmic}[1]\footnotesize
\Require regions $\{X_i\}$, interfaces $\mathcal{I}$ with
$\dim\ge d-1$, query $(q_s,q_g)$, budget $M\gets M_0$
\State $c\gets\infty$;\quad $\mathcal{A}\gets\mathcal{I}$
\Comment{alive interfaces}
\While{time remains}
  \State draw $M$ samples on each $F\in\mathcal{A}$; keep those the
  collision checker accepts
  \State roadmap $\gets$ accepted samples $\cup\,\{q_s,q_g\}$; connect
  nodes that share a region
  \Repeat
    \State $p\gets$ shortest path on the roadmap
    \State $p\gets\textsc{VerifyRepair}(p)$
    \Comment{lazy: Sec.~\ref{sec:certify}}
    \State remove from the roadmap any edge it blacklisted
  \Until{$p$ certified or no path remains}
  \If{$p$ certified \textbf{and} $\mathrm{len}(p)<c$}
    \State $c\gets\mathrm{len}(p)$;\quad $p^\ast\gets p$
    \Comment{certified costs drive the pruning}
  \EndIf
  \State $\mathcal{A}\gets\{F\in\mathcal{A}:\ \ell(F)\le c\}$
  \Comment{lossless: Eq.~\eqref{eq:ell}}
  \State $M\gets 2M$
\EndWhile
\State $p'\gets\textsc{VerifyRepair}(\text{corridor SOCP polish of }
p^\ast)$;\ \Return the cheaper certified of $p^\ast,p'$ with
$\gamma=\mathrm{len}/\mathrm{LB}-1$
\end{algorithmic}
\end{algorithm}

Verification is interleaved, not terminal: a dirty
segment is repaired by a local detour inside its own region, an
unrepairable edge is blacklisted and the search reroutes within the same
round, and only certified costs drive the pruning, since an unverified short path would
tighten the ellipse illegally.  The roadmap holds combinatorially many
paths; one is checked per search, and it usually passes on the first
try because its waypoints were validated before they entered the graph:
one rejection across the whole curated benchmark, a single path's
certification per query even at $3.2\%$ contamination.
% src: e7_panda55.json (one lazy rejection); e16 cert_calls median 201
Validating the roadmap eagerly instead costs two to three orders of
magnitude more, almost all of it on edges no shortest path will ever
use.  The corridor polish runs once, at the end, and is itself
re-verified.

% Certification layer -- v2, 2026-08-08.  Implements VerifyRepair, the
% lazy per-candidate proof of Alg. 1: certificate-ball advancement,
% in-region local detours, blacklist on failure, and the re-verified
% final polish.
\subsection{Continuous clearance certificate}
\label{sec:certify}
\label{sec:layer}

Algorithm~\ref{alg:planner} calls $\textsc{VerifyRepair}$ on every
candidate the search proposes; the subsections that follow implement
it.

A collision checker answers more than yes or no.  A clearance query at
$q$ returns $\phi(q)$, the smallest distance over all $903$ filtered
body pairs, measured on the same collision geometry the boolean checker
uses, so the certificate and the ground truth share one oracle.
Kinematics bounds how fast $\phi$ can fall.  A point on one arm moves
at $\|\dot p\|\le\sum_i r_i|\dot q_i|$, the sum over that arm's
joints, where the reach $r_i$ is the farthest any collision geometry
distal of joint $i$ extends from its axis, and Cauchy-Schwarz turns the
sum into $\sqrt{\sum_i r_i^2}\,\|\dot q\|$.  The same constant
covers arm-arm pairs, whose joint sets concatenate, and self-collision
pairs, whose joint set is a subset.  Hence
$\phi$ is $L$-Lipschitz with $L=\sqrt{\sum_{i=1}^{14} r_i^2}$;
reaches measured on the model give $L\approx 3.4$, the implementation
rounds up to $L=6.0$, and half a million random directional probes of
$\phi$ measure a worst-case rate of $1.01$, a safety factor above
five.  The arms carry no payload in this benchmark; a grasped object
lengthens each $r_i$ by its overhang and changes nothing else.
% src: reach measurement (per arm 1.42,1.22,1.01,0.80,0.61,0.40,0.32 m
% incl. 0.12 m geometry pad -> sqrt sum 3.40); out/e27_lipschitz.json
% src: out/e27_lipschitz.json (sup 1.007, secant max 1.139)

\begin{lemma}[certificate ball]
\label{lem:ball}
If $\phi$ is $L$-Lipschitz and $\phi(q)>0$, then every configuration in the
ball $B\!\left(q,\phi(q)/L\right)$ is collision-free.
\end{lemma}
\begin{proof}
For $\lVert q'-q\rVert<\phi(q)/L$,
$\phi(q')\ge\phi(q)-L\lVert q'-q\rVert>0$.
\end{proof}

One distance query therefore certifies a continuum of configurations, and a
chain of such balls certifies a segment: starting at $a$, query
$\phi$, advance by $\phi/L$ along the segment, and repeat until $b$ is
reached.  The balls overlap by construction, so their union covers the segment,
and the segment is collision-free along its whole length.  There is no resolution parameter: strides lengthen where
clearance is large and shorten near obstacles, so no sliver between two test points is left for the residual.  Certifying a candidate on the 14-DOF library takes about $200$
clearance queries at $2.3$\,ms each, under half a second, mostly near
obstacles where the strides shorten.
% src: e16_headtohead.json (ours_cert_calls median 201); e8_validate.py
The advancement stalls only where $\phi\to 0$, which is precisely a
discovered violation: the certificate and the detector are the same
computation.  The floor $\phi_{\min}=2\times10^{-3}$ that ends a stalled
advancement sits on a plateau: halving or doubling it leaves the certified count unchanged.
% src: out/e24 hmin sweep (14/14/14 at 1e-3..5e-3; 6 at 1e-2)

\subsection{In-region repair}
\label{sec:feedback}

When the advancement reports a dirty stretch on segment $(a,b)$, the
segment does not have to be abandoned.  Both endpoints came from the same region $X$, and convexity makes $X$
a repair kit: for any via point $z\in X$, the detour
$(a,z),(z,b)$ stays inside $X$, so it crosses no region boundary and needs
no new transitions.  The repair samples candidate via points in $X$ around
the dirty stretch, tests each against the collision checker, splices in the
first that certifies, and re-runs the advancement on the modified segment.
The residual is thin, so a legal detour usually exists a few centimeters
away; when the retry budget is exhausted the edge is blacklisted, and
control returns to the search of Algorithm~\ref{alg:planner}, which routes
around it.  
% src: e7_panda55.json (one lazy rejection); e8_validate.py repair_path

Verification produces one fact, that the candidate collides at a
specific configuration $q^\ast$, and its value depends on what the
planner can do with it.  A convex-optimization pipeline has three
options, and none of them is a repair.  Re-solving returns the same
optimum and the same collision; by Proposition~\ref{prop:exploit}, when
the shortcut through the residual is real, every optimum of the program
is invalid.  Adding a constraint that excludes $q^\ast$ moves the
generator's surgery online, one supporting half-space at a time with the
residual curving back around each cut (the surgical remedy of
Sec.~\ref{sec:ciris}).
% src: e8_repair_biman219.json (24,080 cuts, 105/219 budget-exhausted, 19% interfaces lost, 1.25% left)
Deleting the offending region is a blacklist at the wrong
granularity, a corridor for a sliver; eight of the queries our layer solves with an in-region detour have
no alternative corridor.
% src: e16_headtohead.json (rescued pairs)

The structural difference is ownership of the answer.  An optimum is a
global object that feedback reaches only by re-solving a changed problem, with no bound on the rounds.  A sampled candidate is piecewise and the planner holds its
waypoints: feedback lands at the failure site as a spliced via point or a blacklisted edge, the re-search is a millisecond Dijkstra, and a finite graph only shrinks.

\subsection{Corridor polish with re-verification}

The certified path is a chain of straight segments through sampled
crossings, so its cost carries sampling slack.  The region sequence it
visits defines a corridor, and on that corridor the modeled
problem~\eqref{eq:modeled} restricts to a second-order cone program: one
free waypoint per interface, minimize total length.  Solving it moves the crossings to their optimal positions and removes
most of the slack: on the 14-DOF library, raw sampled paths carry $+41.5\%$ over the corridor
optimum and polished paths $+9.6\%$.
% src: e7_biman219.json (raw -> polish)
The polish optimizes the modeled problem, so Proposition~\ref{prop:exploit} applies to it: its
waypoints move toward the model's boundary, residual included.
Algorithm~\ref{alg:planner} therefore sends the polished path back through
$\textsc{VerifyRepair}$ and keeps it only when it is both certified and
cheaper; otherwise the pre-polish certified path stands.

\subsection{Output guarantees}

The planner hands back a path with two guarantees of different kinds.  The clearance certificate is deterministic and covers the
continuum: the
delivered trajectory is collision-free along its entire length.  The
cost certificate is the gap $\gamma$ between the delivered length and a
lower bound $\mathrm{LB}\le m(\lib)$; the implementation uses the
joint-space distance $\lVert q_s-q_g\rVert$, and any relaxation
of~\eqref{eq:modeled} can tighten it.  Over the fifteen certified
benchmark answers $\gamma$ has median $7.1\%$, five sit within $2\%$,
and one attains its bound exactly.
% src: straight-line gamma over center s0 (median 7.1, range 0-50.3);
% relaxation spot checks: pair00 0.930 vs 0.804, pair03 tie, pair10 +0.2%  The two certificates age differently under a contaminated library:

\begin{lemma}[graceful degradation]
\label{lem:immunity}
Any lower bound $\mathrm{LB}\le m(\lib)$ also satisfies
$\mathrm{LB}\le m(\lib\cap\free)$.  A certified path of length $\ell$
therefore satisfies $\ell\le(1+\gamma)\,m(\lib\cap\free)$ with
$\gamma=\ell/\mathrm{LB}-1$, whatever the residual is.
\end{lemma}
\begin{proof}
$\lib\cap\free\subseteq\lib$ shrinks the feasible set
of~\eqref{eq:modeled}, so $m(\lib)\le m(\lib\cap\free)$.
\end{proof}

\noindent
Contamination widens $\gamma$; it cannot falsify it.  An unvalidated geometric cost moves in the opposite direction: it
quotes $m(\lib)$ as the achieved cost of a trajectory that may not be
executable at all.

\paragraph{Implementation}
The planner runs with $M_0{=}1$ center crossing per interface, eight
doubling rounds, certificate floor $\phi_{\min}{=}2\times10^{-3}$,
$L{=}6.0$, and at most $15$ via-point trials over four repair rounds
per segment; the corridor polish and the reference planner's programs
are solved with Clarabel through Drake, the reference in its benchmark
configuration ($60$\,s timeout, $K{=}1$, at most one revisit).  All
timings come from one 32-thread CPU node without GPU.  The task queries
are the $29$ pairs of the benchmark's named grasp configurations that
the library connects; the classical baseline is RRT-Connect with
$0.35$\,rad extension, $0.03$\,rad edge resolution, and $60$\,s per
seed.

% Measurement section -- v1, 2026-08-08.  The residual measured at volume
% and at answer level; the amplification; the three library-level routes
% run to completion.  Labels: sec:measure, sec:ciris.
\section{Experiments}
\label{sec:measure}

This section measures the residual and tests the three predictions
of Section~\ref{sec:problem}, runs the
library-level fixes to completion, and prices the certified planner
against its reference.

\subsection{Libraries and measurement protocol}

Two libraries built with the standard sampling-based generator anchor the
measurement.  The \emph{curated} library covers a 7-DOF arm and a shelf
with $55$ regions grown from hand-picked seeds; the \emph{scaled} library
covers a 14-DOF bimanual scene (two arms,
grippers, shelving, $903$ filtered collision pairs) with $219$ regions grown from $22$ grasp
seeds and $200$ random seeds at the generator's default settings.  Ground
truth throughout is the scene's collision checker, the same oracle the
generator itself sampled.  The instrument: uniform interface samples tested against the oracle.

\subsection{Interface contamination}

% src: e7_panda55.json; e7_biman219.json

Curation works: with seeds placed by hand and time to spare, the
curated library reads $6$ of $2{,}000$ interface samples in collision
($0.3\%$), the regime the family's demonstrations live in.  Scaling
changes that.  Random seeds, default budgets, and an iteration cap that
admits regions with the acceptance test unfinished leave $64$ of
$2{,}000$ in collision ($3.2\%$), above the per-region volume contract,
exactly as the concentration argument of Section~\ref{sec:problem}
predicts.  The number is a property of the configuration, not of one draw (four
regenerations at identical settings: $2.8$ to $3.6\%$), and of dimension (zero at 2 and
3 DOF across a testbed and twelve aerial libraries, $0.15$ to $0.3\%$
at 7, $2.8$ to $3.6\%$ at 14).  The controlled $k$-link family of Remark~\ref{rem:dimlaw}
(Fig.~\ref{fig:dimlaw}) isolates the dependence; both seed budgets
trace the same law, and the optimizer finds what two thousand uniform
samples cannot.

% src: out/e31/b100/k*.json (fixed budget 100 seeds);
% out/e31/ans_k*.json (contained-endpoint answers, dense libraries)
% src: out/e30/uav_*.json (12 worlds, 0/2000 each);
% out/e30/panda_scaled_report.json (3/2000); e15_ciris.json (2-DOF)
% src: out/e24_lib/report_s1..s4.json (3.60/3.10/2.75/3.10%)

\subsection{Answer-level failure}

Volume is the cause; answers are the harm.  We benchmarked the reference
answers of \gcsstar~\cite{chia2024gcsstar} on the scaled library under a
protocol built for certainty: bit-exact reproduction of the stored
cost in the benchmark's exact configuration, the full waypoint sequence
re-derived as the optimum of the convex program, and validation by both
the continuous certificate of Section~\ref{sec:certify} and $2{,}000$
boolean collision probes.
Of the connected task queries, $18$ of $29$ ($62\%$) return colliding
reference answers, against $2$ of $16$ random benchmark queries;
violations reach $39.4$\,mm penetration at the median, $91.3$\,mm at
the deepest, and one path spends $51\%$ of its length in collision.  The verdicts reproduce on a second machine and survive regeneration: four fresh instances return colliding reference answers on $83$ to $84\%$ of the connected pairs drawn from all $22$ grasp configurations ($38/46$, $38/46$, $38/46$, $47/56$).
% src: out/e24_lib/report_s1..s4.json (38/68, 38/68, 38/68, 47/78)
% src: e9_verify_gcsstar.json vs e9_verify_gcsstar_cluster.json

% src: e9_verify_gcsstar.json (2/16); e9_task_sweep_verdicts.json (18/29);
% e9b_depth_profiles.json (depths; all costs reproduced exactly)

Every violation is a single contiguous
mid-path interval with both endpoints clear of obstacles by $38$\,mm or
more: the grasp
configurations are valid, and the path between them is what enters the
shelving.  

\subsection{Library-level remedies}
\label{sec:ciris}

The measurement invites an objection: tighten the generator and the
residual disappears.  We ran the three natural versions of that idea to
completion.

\paragraph{Stricter acceptance contract}
We regenerated the scaled library from the same $222$ seeds with
$\varepsilon=0.1\%$, $\delta=1\%$, four times the particles, ten times
the iterations, and the acceptance test enforced as a gate: a region
that fails is regrown and finally dropped, never admitted unfinished.
The gate accepts $47$ of $222$ seeds, only $3$ of the $22$
shelf-adjacent grasp seeds among them, and the survivors are pairwise
disjoint: zero interfaces against $2{,}757$ in the default library.  We
had registered the opposite prediction.
% src: e10_regen219_report.json (47/222, 0 portals; ~40 node-hours)

\paragraph{Sums-of-squares certification}
C-IRIS~\cite{dai2024ciris} replaces the statistical test with
sums-of-squares proofs.  We composed it with our layer on a 2-DOF
testbed, using the same seven seeds for both generators.  The certificate is real: zero
contamination on $3{,}995$ benchmark samples, against a $17\times$
generation cost.  The certified regions cover $33\%$ of free space
against IrisZo's $83\%$, and overlap in a single pair; at $2.7$ times
the seed budget the certified library still fragments into seven
connected components and answers none of $30$ queries.  
% src: e15_ciris.json; e15b_ciris_dense.json

\paragraph{Uniform margins}
Eroding every region by a margin $\delta$ traces the whole trade-off.
Across nine margins up to $\delta=0.10$\,rad, contamination falls only
from $3.39\%$ to $2.11\%$ while the connected benchmark queries fall
from $16$ to $10$ of $16$: the curve never approaches soundness.  The residual
hides in overlap volumes deeper than any uniform margin the
connectivity survives.
% src: out/e8_erosion_biman219.json (nine-delta sweep)
% src: e8_erosion_biman219.json (16 -> 12 queries connected)

\paragraph{}
A fourth, surgical option, excising discovered residual region by
region, maims the graph while failing to finish
(Sec.~\ref{sec:feedback}).  All four fail the same way, as Section~\ref{sec:problem} predicts.

% Experiments -- v1, 2026-08-08.  Head-to-head on all task queries;
% clean-library performance; what certification costs.  Label: sec:h2h.
\subsection{Comparison on the scaled library}
\label{sec:h2h}

The closing experiment runs both planners on all $29$ connected task queries of the 14-DOF
library: same machine, same library, same queries.  The reference arm is \gcsstar\ in the
benchmark's exact configuration, its answer re-derived as the optimum of
its own program; our arm is Algorithm~\ref{alg:planner} with center crossings.  Both delivered
trajectories face the same judge: the continuous certificate, with dense
boolean probes as the second mechanism.
% src: e16_headtohead.json

\paragraph{Correctness}
The reference planner delivers $21$ trajectories that fail
certification ($18$ probe-confirmed, $3$ on razor-thin margins) and
reports success on every one; ours delivers $15$ certified paths and
$14$ explicit refusals, no invalid answer.  A refusal is information:
every refused pair's reference answer itself fails certification, and thirteen of the fourteen carry certified witnesses from other
planner runs; one defeats every planner tried.  Across ten seeds the certified count moves by at
most one and certified costs vary by a $0.15\%$ median coefficient of
variation.
% src: out/e23_baseline.json; out/e24 seed runs

\paragraph{Speed}
Certification is not a tax; here it is a fourteen-fold speedup.  Our
first certified answer arrives in $0.11$\,s at the median against
$1.59$\,s unverified: center crossings maximize clearance, so the
first candidate's certificate balls are the largest the interfaces
offer, and verification settles in about $200$ clearance queries.

\paragraph{Optimality}
On all eight pairs whose reference answer is physically valid, the
certified cost equals the reference cost to four decimal places; the rescued pairs pay a median $+9.2\%$ (range $+5.1$ to $+50\%$),
the length of the legal detour around an obstacle the reference answer
passes through.

\paragraph{Comparison with sampling-based planning}
Narrow passages are where decomposition planners earn their keep
(Sec.~\ref{sec:related-cert}): a measure-zero target for random sampling is an explicit
interface here.  RRT-Connect with the same ground-truth checker, ten seeds per pair,
solves every pair in raw form, yet only $45\%$ of its
shortcut-smoothed runs survive the continuous certificate; its pipeline
reaches a first certified answer in $10.7$\,s median against our
$0.11$\,s, and on three pairs threading the shelf interior it
certifies zero of thirty runs while the layer delivers.  Where both
certify, its per-seed median cost lands within $1.4\%$ of ours: the
decomposition concedes almost nothing in quality while answering
deterministically, two orders of magnitude faster.  
% src: out/e23_baseline.json (45% cert rate; 10.7s median; 0/30;
% per-seed ratio 0.986)

\paragraph{Curated library}
Where the residual is negligible the layer is simply a fast planner: on
the curated 7-DOF library the first certified incumbent arrives in
$0.13$\,s, $21\times$ before the search baseline's best answer, and
the corridor polish closes the sampling slack from $+5.7\%$ to
$+1.3\%$.
% src: e7_panda55.json; e7_biman219.json
The supplementary video plays reference and certified answers side by
side, with the carried bottles inside the collision model: a 7-DOF arm
moving a bottle between shelves, where the reference sweeps hand and
bottle $68$\,mm through a board; the bimanual arms carrying two bottles
at once and handing one over, with reference violations up to
$50$\,mm; and an aerial benchmark on which the reference optimum flies
through building walls while reporting success and the certified
reroute is $0.9\%$ shorter.
% src: e38 best_d (68.5 mm), e35e/e35h/e35k (48.9/49.7/24.1 mm), e17_uav_compare.mp4 (54.3 m vs 53.79 m)

% Conclusion -- v1, 2026-08-08.
\section{Conclusion}

Planning on graphs of convex sets rests on regions that a sampling-based
generator can only promise probabilistically, and the promise breaks
precisely where optimizers look: a $3.2\%$ interface residual on a scaled
14-DOF library leaves most task-query reference answers physically colliding,
reported as successes.  The cause is structural: cuts cannot exclude
non-convex C-obstacles, finite samples cannot certify a continuum, and an
optimizer concentrates on exactly the error the acceptance test tolerates;
the library-level remedies dissolve the region graph before they deliver soundness.

The planner this paper builds accepts the library as it is and moves
the guarantee to the answer.  Where the model is right the guarantee costs nothing: the certified answer reproduces the reference optimum in a
fourteenth of the reference's time.  Where the model is wrong it returns a
certified detour or an explicit refusal, never a silent violation.

\section*{Acknowledgment}
Claude (Anthropic; Opus~5) was used for editing and
grammar enhancement of the text and assisted in writing and checking
parts of the experiment code.  It generated no result, figure, or
number; every passage and every line of code it touched was reviewed
and verified by the authors.

\def\IEEEbibitemsep{0pt plus .2pt}

\end{document}